\def\year{2021}\relax
\documentclass[letterpaper]{article} 
\usepackage{aaai21}  
\usepackage{times}  
\usepackage{helvet} 
\usepackage{courier}  
\usepackage[hyphens]{url}  
\usepackage{graphicx} 
\usepackage{natbib}  
\usepackage{caption} 
\title{Asking the Right Questions:\\Learning Interpretable Action Models Through Query Answering}
\author {
        Pulkit Verma,
        Shashank Rao Marpally, {\normalfont and}
        Siddharth Srivastava \\
}
\affiliations {
    School of Computing, Informatics, and Decision Systems Engineering\\
    Arizona State University, Tempe, AZ 85281, USA\\
    \{verma.pulkit, smarpall, siddharths\}@asu.edu
}

\usepackage{amssymb,amsmath,amsthm}
\usepackage{pict2e}
\usepackage{subcaption}
\usepackage{algorithm}
\usepackage{algorithmic}
\usepackage{tikz}

\theoremstyle{definition}
\newtheorem{theorem}{Theorem}
\newtheorem{example}{Example}
\newtheorem{definition}{Definition}
\newtheorem{lemma}{Lemma}

\newcommand{\prune}{\langle\rangle}

\newcommand{\dist}{
  \tikz[scale=.2]{\draw[semithick] (0,-1) to ++(360:.4) to ++(0,-1.22) to ++(360:.4);}%
}

\newenvironment{sproof}{
	\proof}{\endproof}

\newcommand{\h}{\mathcal{H}}
\newcommand{\ag}{\mathcal{A}}
\newcommand{\m}{\mathcal{M}}
\newcommand{\Q}{\mathbb{Q}}
\newcommand{\q}{\mathcal{Q}}
\newcommand{\f}{\mathcal{F}}

\newcommand{\mysssection}[1]{\noindent\textbf{#1}\hspace{10pt}}

\begin{document}
\maketitle

\begin{abstract}
    This paper develops a new approach for estimating an
    interpretable, relational model of a black-box
	autonomous agent that can plan and act. Our main
	contributions are a new paradigm for estimating such
	models using a rudimentary query interface with the agent
	and a hierarchical querying algorithm that generates
	an interrogation policy for estimating the agent's
	internal model in a user-interpretable vocabulary.
	Empirical evaluation of our approach shows that despite
	the intractable search space of possible agent models, our
	approach allows correct and scalable estimation of
	interpretable agent models for a wide class of black-box
	autonomous agents. Our results also show that this
	approach can use predicate classifiers to learn
	interpretable models of planning agents that represent
	states as images.
\end{abstract}

\section{Introduction}
\label{sec:introduction}

The growing deployment of AI systems ranging from personal digital
assistants to self-driving cars leads to a pervasive problem: how
would a user ascertain whether an AI system will be safe, reliable, or
useful in a given situation?  This problem becomes particularly
challenging when we consider that most autonomous systems are not
designed by their users; their internal software may be unavailable or
difficult to understand, and it may even change from initial specifications
as a result of learning.  Such scenarios feature \emph{black-box} AI
agents whose models may not be available in terminology that the user
understands. They also show that in addition to developing better AI
systems, we need to develop new algorithmic paradigms for assessing
arbitrary AI systems and for determining the
minimal requirements for AI systems in order to ensure interpretability
and to support such assessments~\cite{srivastava2021unifying}.

This paper presents a new approach for addressing these questions.  It
develops an algorithm for estimating interpretable, relational models
of AI agents by querying them. In doing so, it requires the AI system
to have only a primitive query-response capability to ensure
interpretability. Consider
a situation where Hari(ette) ($\h$) wants a grocery-delivery robot
($\ag$) to bring some groceries, but s/he is unsure whether it is up
to the task and wishes to estimate $\ag$'s internal model in an
interpretable representation that s/he is comfortable with (e.g., a
relational STRIPS-like
language~\cite{Fikes1971,McDermott_1998_PDDL,fox03_pddl}). If $\h$
was dealing with a delivery person, s/he might ask them questions such
as ``would you pick up orders from multiple persons?'' and ``do you
think it would be alright to bring refrigerated items in a regular
bag?'' If the answers are ``yes'' during summer, it would be a cause
for concern. Na\"ive approaches for generating such questions to
ascertain the limits and capabilities of an agent are
infeasible.\footnote{ Just 2 actions and 5 grounded propositions would
yield $7^{2\times5}\sim10^{8}$ possible STRIPS-like models -- each
proposition could be absent, positive or negative in the precondition
and effects of each action, and cannot be positive (or negative) in
both preconditions and effect simultaneously.  A query strategy that
inquires about each occurrence of each proposition would be not only
unscalable but also inapplicable to simulator-based agents that do not
know their actions' preconditions and effects.}

\begin{figure}
    \centering
    \includegraphics[width=\columnwidth]{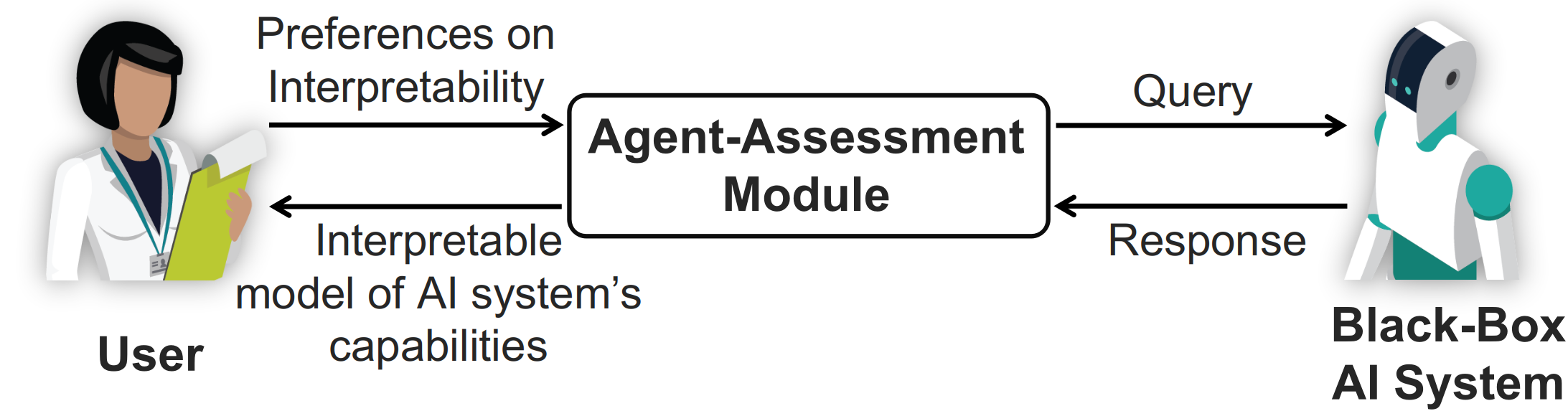}
    \caption{The agent-assessment module uses its user's
    preferred vocabulary, queries the AI system, and delivers a
    user-interpretable causal model of the AI system’s capabilities.
    The AI system does not need to know the user's vocabulary or
    modeling language.}
    \label{fig:aam}
\end{figure}

We propose an agent-assessment module (AAM), shown in
Fig.\,\ref{fig:aam}, which can be connected with an arbitrary AI agent
that has a rudimentary query-response capability: the assessment
module connects $\ag$ with a simulator and provides a sequence of
instructions, or a plan as a \emph{query}. $\ag$ executes the plan in
the simulator and the assessment module uses the simulated outcome as
the response to the query. Thus, given an agent, the assessment module
uses as input: a user-defined vocabulary, the agent's instruction set,
and a compatible simulator. These inputs reflect natural requirements
of the task and are already quite commonly supported: AI systems are
already designed and tested using compatible simulators, and they need
to specify their instruction sets in order to be usable. The user
provides the concepts that they can understand and these concepts can
be defined as functions on simulator states.

In developing the first steps towards this paradigm, we assume that
the user wishes to estimate $\ag$'s internal model as a STRIPS-like
relational model with conjunctive preconditions, add lists, and delete
lists, and that the agent's model is expressible as such. Such models
can be easily translated into interpretable descriptions such as
``under situations where \emph{preconditions} hold, if the agent $\ag$ executes
actions $a_1,\dots, a_k$ it would result in \emph{effects},'' where
preconditions and effects use only the user-provided concepts.
Furthermore, such models can be used to investigate counterfactuals and
support assessments of causality~\cite{Halpern16_causality}.

This fundamental framework (Sec.\,\ref{sec:approach}) can be developed
to support different types of agents as well as various query and
response modalities. E.g., queries and responses could use a speech
interface for greater accessibility, and agents with reliable inbuilt
simulators/lookahead models may not need external simulators. This
would allow AAM to pose queries such as ``what do you think would
happen if you did $\langle$\emph{query plan}$\rangle$'', and the
learnt model would reflect $\ag$'s self-assessment. The ``agent''
could be an arbitrary entity, although the expressiveness of the
user-interpretable vocabulary would govern the scope of the learnt
models and their accuracy. Using AAM with such agents would also help
make them compliant with Level II assistive AI -- systems that make it
easy for operators to learn how to use them
safely~\cite{srivastava2021unifying}.

Our algorithm for the assessment module (Sec.\,\ref{sec:algorithm})
generates a sequence of queries ($\q$) depending on the agent’s
responses ($\theta$) during the query process; the result of the
overall process is a complete model of $\mathcal{A}$. To generate
queries, we use a top-down process that eliminates large classes of
agent-inconsistent
models by computing queries that discriminate between pairs of
\emph{abstract models}. When an abstract model's answer to a query
differs from the agent's answer, we effectively eliminate the entire
set of possible concrete models that are refinements of this abstract
model. Sec.\,\ref{sec:approach} presents our overall framework
with algorithms and theoretical results about their convergence
properties.

Our empirical evaluation (Sec.\,\ref{sec:experiments}) shows that
this method can efficiently learn correct models for black-box
versions of agents using hidden models from the
IPC~\footnote{https://www.icaps-conference.org/competitions}.
It also shows that AAM can use image-based predicate classifiers
to infer correct models for simulator-based agents that respond with
an image representing the result of query plan's execution.

\section{Related Work}
\label{sec:related_work}

A number of researchers have explored the problem of learning agent
models from observations of its
behavior~\cite{gil_94_learning,Yang2007,Cresswell09,Zhuo13action}.
Such action-model learning approaches have also found practical
applications in robot navigation~\cite{balac_2000_learning}, player
behavior modeling~\cite{Krishnan_towards_2020}, etc. To the best of
our knowledge, ours is the first approach to address the problem of
generating query strategies for inferring relational models of
black-box agents.

\citet{Amir08} use logical filtering~\cite{Amir03} to learn partially
observable action models from the observation traces.
LOCM~\cite{Cresswell09} and LOCM2~\cite{Cresswell11} present another
class of algorithms that use finite-state machines to create action
models from observed plan traces. \citet{Camacho19} present an approach
for learning highly expressive LTL models from an agent's observed state
trajectories using an oracle with knowledge of the target LTL
representation. This oracle can also generate counterexamples when the
estimated model differs from the true model. In contrast, our approach
does not require such an oracle. Also, unlike \citet{Stern2017},
our approach does not need intermediate states in execution traces.
In contrast to approaches for white-box model maintenance~\cite{Bryce2016},
our approach does not require $\ag$ to know about $\h$'s preferred
vocabulary.

LOUGA~\cite{Kucera18} combines a genetic algorithm with an ad-hoc
method to learn planning operators from observed plan traces.
FAMA~\cite{aineto2019learning} reduces model recognition to a planning
problem and can work with partial action sequences and/or state traces
as long as correct initial and goal states are provided. While both
FAMA and LOUGA require a postprocessing step to update the learnt
model's preconditions to include the intersection of all states where
an action is applied, it is not clear that such a process would
necessarily converge to the correct model. Our experiments indicate
that such approaches exhibit oscillating behavior in terms of model
accuracy because some data traces can include spurious predicates, which
leads to spurious preconditions being added to the model's actions. 
FAMA
also assumes that there are no negative literals in action preconditions.

\citet{bonet20learning} present an algorithm for learning relational
models using a SAT-based method when the action schema, predicates, etc.
are not available. This approach takes as input a predesigned correct
and complete directed graph encoding the structure of the entire state
space. The authors note that their approach is viable for problems with
small state spaces. While our method provides an end-to-end solution,
it can also be used in conjunction with such approaches to create the
inputs they need. \citet{Khardon96} address the problem of making
model-based inference faster \emph{given a set of queries}, under the
assumption that a static set of models represents the true knowledge
base.

In contrast to these directions of research, our approach directly
queries the agent and is guaranteed to converge to the true model
while presenting a running estimate of the accuracy of the derived
model; hence, it can be used in settings where  the agent’s model
changes due to learning or a software update. In such a scenario, our
algorithm can restart to query the system, while approaches that derive
models from observed plan traces would require arbitrarily long
data collection sessions to get sufficient
uncorrelated data.

Incremental Learning Model~\cite{ng2019incremental}
uses reinforcement learning to learn a nonstationary model without
using plan traces, and requires extensive training to learn the
full model correctly. \citet{chitnis2021glib} present an approach
for learning probabilistic relational models where they use goal sampling
as a heuristic for generating relevant data, while we reduce that problem
to query synthesis using planning. Their approach is shown to work well
for stochastic environments, but puts a much higher burden on the AI system
for inferring its model.
This is because the AI system has to generate a conjunctive goal formula while
maximizing exploration, find a plan to reach that goal, and correct the
model as it collects observations while executing the plan.

The field of active learning~\cite{settles12} addresses the related
problem of selecting which data-labels to acquire for learning
single-step decision-making models using statistical measures of
information. However, the effective feature set here is the set of
all possible plans, which makes conventional methods for evaluating
the information gain of possible feature labelings infeasible.
In contrast, our approach uses a hierarchical abstraction
to select queries to ask, while inferring a multistep
decision-making (planning) model. Information-theoretic metrics could
also be used in our approach whenever such information is available.

\section{The Agent-Interrogation Task}
\label{sec:approach}

We assume that $\h$ needs to estimate $\ag$'s model as a STRIPS-like
planning model
represented as a pair
$\mathcal{M} = \langle \mathbb{P}, \mathbb{A} \rangle$, where
$\mathbb{P} = \{p_1^{k_1},\dots, p_n^{k_n} \}$ is a finite set of
predicates with arities $k_i$; $\mathbb{A} = \{a_1,\dots, a_k \}$
is a finite set of parameterized actions (operators). Each action
$a_j \in \mathbb{A}$ is represented as a tuple $\langle header(a_j),
pre(a_j), \emph{eff}(a_j) \rangle $, where $header(a_j)$ is the action
header consisting of action name and action parameters, $pre(a_j)$
represents the set of predicate atoms that must be true in a state
where $a_j$ can be applied, $\emph{eff}(a_j)$ is the set of positive
or negative predicate atoms that will change to true or false
respectively as a result of execution of the action $a_j$. Each
predicate can be instantiated using the parameters of an action, where
the number of parameters are bounded by the maximum arity of the
action. E.g., consider the action $\emph{load\_truck}(?v1, ?v2, ?v3)$
and predicate $at(?x, ?y)$ in the IPC Logistics domain. This predicate
can be instantiated using action parameters $?v1$, $?v2$, and $?v3$ as
$at(?v1, ?v1)$, $at(?v1, ?v2)$, $at(?v1, ?v3)$, $at(?v2, ?v2)$,
$at(?v2, ?v1)$, $at(?v2, ?v3)$, $at(?v3, ?v3)$, $at(?v3, ?v1)$,
and $at(?v3, ?v2)$. We represent the set of all such possible
predicates instantiated with action parameters as $\mathbb{P}^*$.

AAM uses the following information as input. It receives its instruction
set in the form of $header(a)$ for each $a \in \mathbb{A}$ from the agent.
AAM also receives a predicate vocabulary
$\mathbb{P}$ from the user with functional definitions of each predicate.
This gives AAM sufficient information to perform a dialog with $\ag$ about
the outcomes of hypothetical action sequences.

We define the overall problem of agent interrogation as follows. Given
a class of queries and an agent with an unknown model which can answer
these queries, determine the model of the agent. More precisely, an
\emph{agent interrogation task} is defined as a tuple $\langle
\mathcal{M}^\mathcal{A}, \mathbb{Q}, \mathbb{P}, \mathbb{A}_H\rangle$,
where $\mathcal{M}^\mathcal{A}$ is the true model (unknown to AAM) of
the agent $\ag$  being interrogated, $\mathbb{Q}$ is the class of
queries that can be posed to the agent by AAM, and $\mathbb{P}$ and
$\mathbb{A}_H$ are the sets of predicates and action headers that AAM
uses based on inputs from $\h$ and $\mathcal{A}$. The objective of
the agent interrogation task is to derive the agent model $\m^\ag$
using $\mathbb{P}$ and $\mathbb{A}_H$. Let $\Theta$ be the set of
possible answers to queries. Thus, strings $\theta^* \in \Theta^*$
denote the information received by AAM at any point in the query
process. Query policies for the agent interrogation task are functions
$\theta^*\rightarrow \mathbb{Q}\cup \{\emph{Stop}\}$ that map sequences
of answers to the next query that the interrogator should ask. The
process stops with the \emph{Stop} query. In other words, for all
answers $\theta \in \Theta$, all valid query policies map all sequences
$x\theta$ to \emph{Stop} whenever $x\in \Theta^*$ is mapped to
\emph{Stop}. This policy is computed and executed online.

\subsubsection{Components of agent models} In order to formulate our
solution approach, we consider a model $\mathcal{M}$ to be comprised
of components called \emph{palm} tuples of the form $\lambda = \langle
p, a, l, m \rangle$, where $p$ is an instantiated predicate from the
vocabulary $\mathbb{P}^*$; $a$ is an action from the set of
parameterized actions $\mathbb{A}$, $l \in \{\emph{pre}, \emph{eff}\}$,
and $m \in \{+, - , \emptyset\}$. For convenience, we use the subscripts
$p,a,l,$ or $m$ to denote the corresponding component in a palm tuple.
The presence of a palm tuple $\lambda$ in a model denotes the fact that
in that model, the predicate $\lambda_p$ appears in an action $\lambda_a$
at a location $\lambda_l$ as a true (false) literal when sign $\lambda_m$
is positive (negative), and is absent when $\lambda_m = \emptyset$. This
allows us to define the set-minus operation $M \setminus \lambda$ on this
model as removing the palm tuple $\lambda$ from the model.
We consider two palm tuples $\lambda_1 = \langle p_1, a_1, l_1, m_1
\rangle$ and $\lambda_2 = \langle p_2, a_2, l_2, m_2 \rangle$ to be
\emph{variants} of each other ($\lambda_1 \sim \lambda_2$) iff they
differ only on mode $m$, i.e., $\lambda_1 \sim \lambda_2 \Leftrightarrow
(\lambda_{1_p} = \lambda_{2_p}) \land (\lambda_{1_a} = \lambda_{2_a})
\land (\lambda_{1_l} = \lambda_{2_l}) \land (\lambda_{1_m} \neq
\lambda_{2_m})$. Hence, mode assignment to a \textit{pal} tuple
$\gamma = \langle p, a, l \rangle$ can result in 3 palm tuple variants
$\gamma^+ = \langle p, a, l, +\rangle$, $\gamma^- = \langle p, a, l,
-\rangle$, and $\gamma^\emptyset = \langle p, a, l,\emptyset\rangle$.

\subsubsection{Model abstraction} We now define the notion of abstraction
used in our solution approach. Several approaches have explored the use
of abstraction in planning~\cite{Sacerdoti74,Giunchiglia92,Helmert07,
Backstrom13,Srivastava16}. The definition of abstraction used in this
work extends the concept
of predicate and propositional domain abstractions~\cite{Srivastava16}
to allow for the projection of a single \emph{palm} tuple $\lambda$.

An abstract model is one in which all variants of
at least one pal tuple are absent. Let $\Lambda$ be the set of all
possible palm tuples which can be generated using a predicate
vocabulary $\mathbb{P}^*$ and an action header set $\mathbb{A}_H$.
Let $\mathcal{U}$ be the set of all consistent (abstract and
concrete) models that can be expressed as subsets of $\Lambda$,
such that no model has multiple variants of the same palm tuple.
We define abstraction of a model as:

\begin{definition}
\label{def:abstraction}
    The \emph{abstraction of a model} $\m$ with
    respect to a palm tuple $\lambda \in \Lambda$, is
    defined by $f_\lambda: \mathcal{U} \rightarrow
    \mathcal{U}$ as $f_\lambda (\m) = \m \setminus \lambda$.
\end{definition}

\begin{figure}[t]
    \centering
    \includegraphics[width=\columnwidth]{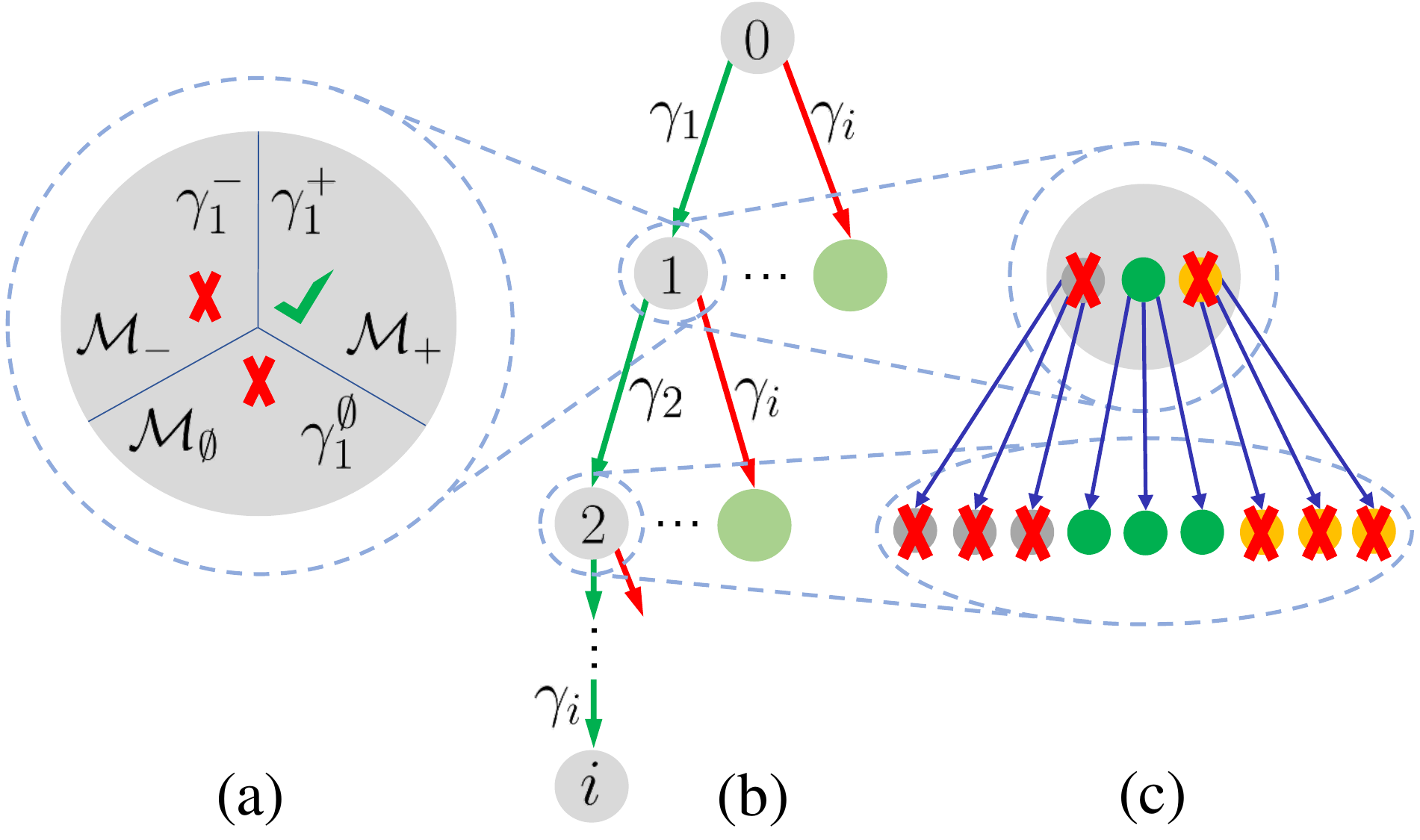}
    \caption{(b) Lattice segment explored in random order of
    $\gamma_i \in \Gamma$; (a) At each node, 3 abstract models are
    generated and 2 of them are discarded based on query responses;
    (c) An abstract model rejected at any level is equivalent to
    rejecting 3 models at the level below, 9 models two levels down,
    and so on.}
    \label{fig:lattice_iaa}
\end{figure}

We extend this notation to define the abstraction of a set of models $\m$
with respect to a palm tuple $\lambda$ as $X = \{f_\lambda (m): m \in \m\}$.
We use this abstraction framework to define a subset-lattice over
abstract models (Fig.\,\ref{fig:lattice_iaa}(b)). Each node in the
lattice represents a collection of possible abstract models which are
possible variants of a pal tuple $\gamma$. E.g., in the node labeled 1
in Fig.\,\ref{fig:lattice_iaa}(b), we have models corresponding to
$\gamma_1^+$, $\gamma_1^-$, and $\gamma_1^\emptyset$. Two nodes in the
lattice are at the same level of abstraction if they contain the same
number of pal tuples. Two nodes $n_i$ and $n_j$ in the lattice are
connected if all the models at $n_i$ differ with all the models in $n_j$
by a single palm tuple. As we move up in the lattice following these
edges, we get more abstracted versions of the models, i.e., containing
less number of pal tuples; and we get more concretized models, i.e.,
containing more number of pal tuples, as we move downward. We now define
this model lattice:

\begin{definition}
\label{def:lattice}
	A \textbf{\emph{model lattice $\mathcal{L}$}} is a 5-tuple
	$\mathcal{L} = \langle N, E, \Gamma, \ell_N, \ell_E \rangle$, where
    $N$ is a set of lattice nodes,
    $\Gamma$ is the set of all pal tuples $\langle p, a, l\rangle$,
    $\ell_N: N \rightarrow 2^{2^\Lambda}$ is a node label function where
    $\Lambda = \Gamma \times \{+,-,\emptyset\}$ is the set of all palm
    tuples, $E$ is the set of lattice edges, and
    $\ell_E: E \rightarrow \Gamma$ is a function mapping edges to edge
    labels such that for each edge $n_i \rightarrow n_j,
    \,\ell_N(n_j) = \{\xi \cup \{ \gamma^k\}|\, \xi \in \ell_N(n_i),
    \gamma = \ell_E(n_i \rightarrow n_j), k \in\{+,-,\emptyset\}\}$,
    and $\ell_N(\top) = \{\phi\}$ where $\top$ is the supremum
    containing the empty model $\phi$.
\end{definition}

A node $n \in N$ in this lattice $\mathcal{L}$ can be
uniquely identified by the sequence of pal tuples that label the edges
leading to it from the supremum. As shown in Fig.\,\ref{fig:lattice_iaa}(a),
even though theoretically $\ell_N: N \rightarrow 2^{2^\Lambda}$, not all the
models are stored at any node as at least one is pruned out based on some
query $\q$ $ \in \mathbb{Q}$. Additionally, in these model lattices,
every node has an edge going out from it corresponding to each pal tuple
that is not present in the paths leading to it from the most abstracted
node. At any stage during the interrogation, nodes in such a lattice are
used to represent the set of possible models given the agent's
responses up to that point. At every step, our 
algorithm
creates queries online that help us determine the next descending edge to take
from a lattice node; corresponding to the path $0,\dots,i$ in
Fig.\,\ref{fig:lattice_iaa}(b). This also avoids
generating and storing the complete lattice, which can be doubly
exponential in number of predicates and actions.

\subsubsection{Form of agent queries} As discussed earlier, based on $\ag$'s
responses $\theta$, we pose queries to the agent and infer $\ag$'s model. We
express queries as functions that map models to answers. Recall that
$\mathcal{U}$ is the set of all possible (concrete and abstract) models,
and $\Theta$ is the set of possible responses. A \emph{query} $\mathcal{Q}$
is a function $\mathcal{Q}: \mathcal{U} \rightarrow \Theta$.

In this paper, we utilize only one class of queries: \emph{plan outcome
queries} ($\mathcal{Q}_{PO}$), which are parameterized by a state
$s_I$ and a plan $\pi$. Let $P$ be the set of predicates $\mathbb{P}^*$
instantiated with objects $O$ in an environment. $\mathcal{Q}_{PO}$
queries ask $\ag$ the length of the longest prefix of the plan $\pi$
that it can execute successfully when starting in the state $s_\mathcal{I}
\subseteq P$ as well as the final state $s_\mathcal{F} \subseteq P$ that this
execution leads to. E.g., ``Given that the truck $t1$ and package $p1$
are at location $l1$, what would happen if you executed the plan $\langle
load\_truck(p1,t1,l1)$, $drive(t1,l1,l2)$, $unload\_truck(p1,t1,l2)
\rangle$?''

A response to such queries can be of the form ``I can execute the plan
till step $\ell$ and at the end of it $p1$ is in truck $t1$ which is at
location $l1$''. Formally, the response $\theta_{PO}$ for plan outcome
queries is a tuple $\langle \ell, s_{\mathcal{F}} \rangle$, where $\ell$
is the number of steps for which the plan $\pi$ could be executed, and
$s_{\mathcal{F}} \subseteq P$ is the final state after executing $\ell$
steps of the plan. If the plan $\pi$ cannot be executed fully according
to the agent model $\mathcal{M}^\mathcal{A}$ then $\ell < len(\pi)$,
otherwise $\ell = len(\pi)$. The final state $s_{\mathcal{F}}\subseteq P$
is such that $\mathcal{M}^\ag \models \pi{[1:\ell]}(s_{\mathcal{I}}) =
s_{\mathcal{F}}$, i.e., starting with a state $s_{\mathcal{I}}$, $\m^\ag$
successfully executed first $\ell$ steps of the plan $\pi$. Thus,
$\mathcal{Q}_{PO}: \mathcal{U} \rightarrow \mathbb{N} \times 2^P$,
where $\mathbb{N}$ is the set of natural numbers.

Not all queries are useful, as some of them might not increase our
knowledge of the agent model at all. Hence, we define some properties
associated with each query to ascertain its usability.
A query is \emph{useful} only if it can distinguish between two models. More
precisely, a query $\mathcal{Q}$ is said to \emph{distinguish} a pair
of models $\mathcal{M}_i$ and $\mathcal{M}_j$, denoted as
$\mathcal{M}_i \dist^{\mathcal{Q}} \mathcal{M}_j$, iff
$\mathcal{Q}(\mathcal{M}_i) \neq \mathcal{Q}(\mathcal{M}_j)$.

\begin{definition}
Two models $\m_i$ and $\m_j$ are said to be \emph{distinguishable},
denoted as $\m_i \dist \,\m_j$, iff there exists a query that can
distinguish between them, i.e., $\exists \q\,\,\m_i \dist^{\q} \m_j$.
\end{definition}

Given a pair of abstract models, we wish to determine whether one of
them can be pruned, i.e., whether there is a query for which at least one
of their answers is inconsistent with the agent's answer. Since this is
computationally expensive to determine, and we wish to reduce the number
of queries made to the agent, we first evaluate whether the two models
can be distinguished by any query, independent of consistency of their
response with that of the agent. If the models are not distinguishable,
it is redundant to try to prune one of them under the given query class.

Next, we determine if at least one of the two distinguishable models is
consistent with the agent. When comparing the responses of two models at
different levels of abstraction, we must consider the fact that the agent's
response may be at a different level of abstraction if the given pair
of models is abstract. Taking this into account, we formally define what
it means for an abstract model $\m_i$'s response to be consistent
with that of agent model $\m^\ag$:

\begin{definition}
\label{def:consistency}
Let $\q$ be a query such that $\m_i \dist^{\q} \m_j$;
$\q(\m_i) = \langle \ell^i,\langle p^i_1,\dots, p^i_m \rangle \rangle$,
$\q(\m_j) = \langle \ell^j,\langle p^j_1,\dots, p^j_n \rangle \rangle$, and
$\q(\m^\ag) = \langle \ell^\ag, \langle p^\ag_1,\dots,p^\ag_k\rangle \rangle$.
$\m_i$'s response to $\q$ is \emph{consistent} with that of $\m^\ag$, i.e.,
$\q(\m^\ag) \models \q(\m_i)$ if $\ell^\ag = len(\pi^\q)$, $len(\pi^\q) =
\ell^i$ and $\{ p^i_1,\dots, p^i_m \} \subseteq \{ p^\ag_1,\dots,p^\ag_k\}$.
\end{definition}

Using this notion of consistency, we can now reason that given a set of
distinguishable models $\m_i$ and $\m_j$, and their responses in addition to
the agent's response to the distinguishing query, the models are prunable if
and only if exactly one of their responses is consistent with that of the
agent. Formally, we define prunability as:

\begin{definition}
\label{def:prunable}
Given an agent-interrogation task $\langle \m^\ag, \Q,\mathbb{P}, \mathbb{A}_H
\rangle$, two models $\m_i$ and $\m_j$ are \emph{prunable}, denoted as
$\m_i \prune \m_j$, iff  $\exists \mathcal{Q}\in \mathbb{Q}:
\m_i \dist^\mathcal{Q} \m_j $ $\land$ $(\mathcal{Q}(\m^\ag) \models
\mathcal{Q}(\m_i)$ $\land$ $\mathcal{Q}(\m^\ag) \not\models \mathcal{Q}(\m_j))$
$\lor$ $(Q(\m^\ag) \not\models \mathcal{Q}(\m_i)$ $\land$ $\mathcal{Q}(\m^\ag)
\models \mathcal{Q}(\m_j))$.
\end{definition}

\subsection{Solving the Interrogation Task}
\label{sec:algorithm}

We now discuss how we solve the agent interrogation task by incrementally
adding palm variants to the class of abstract models and pruning out
inconsistent models by generating distinguishing queries.

\begin{example}
    Consider the case of a delivery agent. Assume that AAM is considering
    two abstract models $\m_1$ and $\m_2$ having only one action
    $load\_truck(?p,?t,?l)$ and the predicates $at(?p,?l)$, $at (?t,?l)$,
    $in (?p,?t)$, and that the agent's model is $\m^\ag$
    (Fig.\,\ref{fig:load_truck}). AAM can ask the agent what will happen
    if $\ag$ loads package $p1$ into truck $t1$ at location $l1$ twice. The
    agent would respond that it could execute the plan only till length
    $1$, and the state at the time of this failure would be $at(t1,l1)
    \land in(p1,t1)$.
\end{example}

\begin{figure}[t]
    \small
    (a) $\m^\ag$\textbf's $\texttt{load\_truck(?p,?t,?l)}$ action (unknown to $\h$)\\
    {
        \centering
        \texttt{
            \begin{tabular}{|p{0.35\columnwidth} p{0.09\columnwidth} p{0.40\columnwidth}|}
                \hline
                at(?t,?l), at(?p,?l) & $\longrightarrow$ & in(?p,?t), $\neg$(at(?p,?l))\\
                \hline
            \end{tabular}
        }
    }
    \\
    (b) $\mathcal{M}_1$'s  $\texttt{load\_truck(?p,?t,?l)}$ action\\
    {
        \centering
        \texttt{
            \begin{tabular}{|p{0.35\columnwidth} p{0.09\columnwidth} p{0.40\columnwidth}|}
                \hline
                at(?t,?l), at(?p,?l) & $\longrightarrow$ & in(?p,?t)\\
                \hline
            \end{tabular}
        }
    }
    \\
    (c) $\mathcal{M}_2$'s  $\texttt{load\_truck(?p,?t,?l)}$ action\\
    {
        \centering
        \texttt{
            \begin{tabular}{|p{0.35\columnwidth} p{0.09\columnwidth} p{0.40\columnwidth}|}
                \hline
                at(?t,?l) & $\longrightarrow$ & in(?p,?t)\\
                \hline
            \end{tabular}
        }
    }
    \\
    (d) $\mathcal{M}_3$'s $\texttt{load\_truck(?p,?t,?l)}$ action\\
    {
        \centering
        \texttt{
            \begin{tabular}{|p{0.35\columnwidth} p{0.09\columnwidth} p{0.40\columnwidth}|}
                \hline
                at(?t,?l) & $\longrightarrow$ &  ()\\
                \hline
            \end{tabular}
        }
    }
    \caption{$\emph{load\_truck}$ actions of the agent model $\m^\ag$ and 
    3 abstracted models $\m_1$, $\m_2$, and $\m_3$. Here $X\rightarrow Y$
    means that $X$ is the precondition of an action and $Y$ is the effect.}
    \label{fig:load_truck}
\end{figure}

Algorithm \ref{alg:AIA} shows AAM's overall algorithm. It takes the agent
$\ag$, the set of instantiated predicates $\mathbb{P}^*$, the set of all
action headers $\mathbb{A}_H$, and a set of random states $\mathbb{S}$ as
input, and gives the set of functionally equivalent estimated models
represented by $poss\_models$ as output. $\mathbb{S}$ can be generated in
a preprocessing step given $\mathbb{P}^*$. AIA initializes $poss\_models$
as a set consisting of the empty model $\phi$ (line 3)
representing that AAM is starting at the supremum $\top$ of the model
lattice.

In each iteration of the main loop (line 4), AIA maintains an abstraction
lattice and keeps track of the current node in the lattice. It picks a
pal tuple $\gamma$ corresponding to one of the descending edges in the
lattice from a node given by some input ordering of $\Gamma$. The
correctness of the algorithm does not depend on this ordering. It then
stores a temporary copy of $poss\_models$ as $new\_models$ (line 5) and
initialize an empty set at each node to store the pruned models (line 6).

The inner loop (line 7) iterates over the set of all possible abstract
models that AIA has not rejected yet, stored as $new\_models$. It then
loops over pairs of modes (line 8), which are later used to generate
queries and refine models. For the chosen pair of modes,
\textit{generate\_query()} is called (line 9) which returns two models
concretized with the chosen modes and a query $\q$ which can
distinguish between them based on their responses.

AIA then calls \textit{filter\_models()} which poses the query $\q$ to
the agent and the two models. Based on their responses, AIA prunes the
models whose responses are not consistent with that of the agent
(line 11). Then it updates the estimated set of possible models
represented by $poss\_models$ (line 18).

\begin{algorithm}[tb]

   \caption{Agent Interrogation Algorithm (AIA)}
   \label{alg:AIA}
\begin{algorithmic}[1]
    \STATE {\bfseries Input:} $\ag, \mathbb{A}_H, \mathbb{P}^*, \mathbb{S}$
    \STATE {\bfseries Output:} poss\_models
    \STATE Initialize poss\_models $=$ \{$\phi$\}
    \FOR {$\gamma$ in some input pal ordering $\Gamma$}
        \STATE new\_models $\gets$ poss\_models
        \STATE pruned\_models$=$ \{\}
		\FOR {each $\m'$ in new\_models}

			\FOR{each pair $\{i, j\}$ in $\{+,-,\emptyset\}$}
			    \STATE $\q$, $\m_i$, $\m_j$ $\gets$ generate\_query($\m', i, j, \gamma, \mathbb{S}$)
				\STATE $\m_{prune}$ $ \gets$filter\_models($\q,\m^\ag\!,\m_i, \m_j\!$)
				\STATE pruned\_models$\gets$ pruned\_models $\cup \, \m_{prune}$
			\ENDFOR
		\ENDFOR
		\IF {pruned\_models is $\emptyset$}
            \STATE update\_pal\_ordering($\Gamma, \mathbb{S}$)
            \STATE continue
		\ENDIF
        \STATE poss\_models $\gets$ new\_models $\times\, \{\gamma^+, \gamma^-, \gamma^\emptyset\}$ $\setminus$\\ \qquad \qquad \qquad \qquad pruned\_models
	\ENDFOR
\end{algorithmic}
\end{algorithm}

If AIA is unable to prune any model at a node (line 14), it modifies the
pal tuple ordering (line 15). AIA continues this process until it reaches
the most concretized node of the lattice (meaning all possible palm
tuples $\lambda \in \Lambda$ are refined at this node). The remaining
set of models represents the estimated set of models for $\ag$. The
number of resolved palm tuples can be used as a running estimate of
accuracy of the derived models. AIA requires $O(|\mathbb{P}^*| \!\! \times\!\!
|\mathbb{A}|)$ queries as there are $ 2 \times|\mathbb{P}^*| \!\!\times \!|\mathbb{A}|$
pal tuples. However, our empirical studies show that we never generate so
many queries.

\subsection{Query Generation}
\label{sec:query_gen}

The query generation process corresponds to the \textit{generate\_query()}
module in AIA which takes a model $\m'$, the pal tuple $\gamma$, and 2
modes $i,j \in \{+,-,\emptyset\}$ as input; and returns the models
$\m_i = \m' \cup \{\gamma^i\}$ and $\m_j = \m' \cup \{\gamma^j\}$, and a
plan outcome query $\q$ distinguishing them, i.e., $\m_i \dist^\q \m_j$.

\begin{algorithm}[tb]
   \caption{Query Generation Algorithm}
   \label{alg:query}
\begin{algorithmic}[1]
    \STATE {\bfseries Input:} $\m', i, j, \gamma, \mathbb{S}$
    \STATE {\bfseries Output:} $\q, \m_i, \m_j$
    \STATE $\m_i,\m_j \gets$ add\_palm($\m',i,j, \gamma$)
    \FOR {$s_\mathcal{I}$ in $\mathbb{S}$}
        \STATE dom, prob $\gets$ get\_planning\_prob ($s_\mathcal{I},\m_i, \m_j$)
        \STATE $\pi$ $\gets$ planner(dom, prob)
        \STATE $\q \gets \langle s_\mathcal{I}, \pi \rangle$
        \STATE \textbf{if} $\pi$ \textbf{then} break \textbf{end if}
    \ENDFOR
    \RETURN $\q$, $\m' \cup \{\gamma^i\}$, $\m' \cup \{\gamma^j\}$
\end{algorithmic}
\end{algorithm}

Plan outcome queries have 2 components, an initial state $s_\mathcal{I}$ and a plan
$\pi$. AIA gets $s_\mathcal{I}$ from the input set of random states $\mathbb{S}$ (line
4). Using $s_\mathcal{I}$ as the initial state, the idea is to find a plan, which when
executed by $\m_i$ and $\m_j$ will lead them either to different states, or
to a state where only one of them can execute the plan further. Later we pose
the same query to $\ag$ and prune at least one of $\m_i$ and $\m_j$. Hence, we
aim to prevent the models inconsistent with the agent model $\m^\ag$ from
reaching the same final state as $\m^\ag$ after executing the query $\q$ and
following a different state trajectory. To achieve this, we reduce the problem
of generating a plan outcome query from $\m_i$ and $\m_j$
into a planning problem.

The reduction proceeds by creating temporary models $\m_i''$ and $\m_j''$.
We now discuss how to generate them. We add the pal tuple
$\gamma = \langle p,a,l \rangle$ in modes $i$ and $j$ to $\m'$ to get $\m_i'$
and $\m_j'$, respectively.
If the location $l = \emph{eff}$, we add the palm tuple normally
to $\m'$, i.e., $\m_m' = \m' \cup \langle p , a, l, m \rangle$, where
$m \in \{i,j\}$. If $l = \emph{pre}$, we add a dummy predicate
$p_u$ in disjunction with the predicate $p$ to the precondition of both
the models.
We then modify the models $\m_i'$ and $\m_j'$
further in the following way:
\begin{align*}
    \m_m'' = \m_m' \cup \, &\{\langle p_u, a', l',+ \rangle : \forall a',l' \, \, \langle a',l' \rangle \not\in \\ &\{ \!\langle a^*\!\!,l^*\!\rangle\!\!: \exists m^* \, \langle p,a^*,l^*,m^*\rangle \in \m'\}\}\\
    \cup \, &\{\langle p_u, a', l',- \rangle : \forall a',l' \, \, \langle a',l' \rangle \in \\
    &\{ \!\langle a^*\!\!,l^*\!\rangle\!\!: \!l^*\!\!= \!\emph{eff} \land\! \exists m^* \, \!\langle p,a^*,l^*,m^*\rangle \!\!\in\! \m'\}\}
\end{align*}

$p_u$ is added only for generating a distinguishing query and is not part
of the models $\m_i$ and $\m_j$ returned by the query generation process.
Without this modification, an inconsistent abstract model may have a
response consistent with $\ag$.

We now show how to reduce plan outcome query generation into a planning
problem $P_{PO}$ (line 5). $P_{PO}$ uses conditional effects in its actions
(in accordance with PDDL~\cite{McDermott_1998_PDDL,fox03_pddl}). The model used to define
$P_{PO}$ has predicates from both models $\m_i''$ and $\m_j''$
represented
as $\mathcal{P}^{\m_i''}$ and $\mathcal{P}^{\m_j''}$ respectively,
in addition
to a new dummy predicate $p_\psi$. The action headers are the same as
$\mathbb{A}_H$. Each action's precondition is a disjunction of the
preconditions of $\m_i''$ and $\m_j''$. This makes an action applicable in a
state $s$ if either $\m_i''$ or $\m_j''$ can execute it in $s$. The effect of
each action has 2 conditional effects, the first applies the effects of both
$\m_i''$ and $\m_j''$'s action if preconditions of both $\m_i''$ and $\m_j''$ are
true, whereas the second makes the dummy predicate $p_\psi$ true if precondition of
only one of $\m_i''$ and $\m_j''$ is true. Formally, we express this planning
problem as $P_{PO} = \langle \m^{PO}, s_{\mathcal{I}}, G \rangle$, where
$\m^{PO}$ is a model with predicates $\mathbb{P}^{PO}= \mathcal{P}^{\m_i''}
\cup \mathcal{P}^{\m_j''} \cup p_{\psi}$, and actions $\mathbb{A}^{PO}$
where for each action $a \in \mathbb{A}^{PO}$,
$\emph{pre}(a) = \emph{pre}(a^{\m_i''}) \lor
\emph{pre}(a^{\m_j''})$ and $\emph{eff}(a)$ =
\begin{align*}
    (\emph{when}\, &(\emph{pre}(a^{\m_i''}) \land \emph{pre}(a^{\m_j''}))
    (\emph{eff}\,(a^{\m_i''}) \land \emph{eff}\,(a^{\m_j''})))\\
    (\emph{when}\, &((\emph{pre}(a^{\m_i''}) \land \neg \emph{pre}(a^{\m_j''})) \lor \\
    &(\neg \emph{pre}(a^{\m_i''}) \land    \emph{pre}(a^{\m_j''}))) \,(p_\psi)),
\end{align*}
The initial state $s_{\mathcal{I}} = s_{\mathcal{I}}^{\m_i''} \land
s_{\mathcal{I}}^{\m_j''}$, where $s_{\mathcal{I}}^{\m_i''}$ and
$s_{\mathcal{I}}^{\m_j''}$ are copies of all predicates in
$s_{\mathcal{I}}$, and $G$ is the goal formula expressed as
$\exists p \,\, (p^{\m_i''} \land \neg p^{\m_j''}) \lor
(\neg  p^{\m_i''} \land p^{\m_j''}) \lor p_\psi$.

With this formulation, the goal is reached when an action in $\m_i''$
and $\m_j''$ differs in either a precondition (making only one of them
executable in a state), or an effect (leading to different final states on
applying the action). E.g., consider the models with differences in
$\emph{load\_truck}(p1,t1,l1)$ as shown in Fig.\,\ref{fig:load_truck}.
From the state $at(t1,l1) \land \neg at(p1,l1)$, $\m_2$ can execute
$\emph{load\_truck}(p1,t1,l1)$ but $\m_1$ cannot. Similarly, in state
$at(t1,l1) \land at(p1,l1)$, executing $\emph{load\_truck}(p1,t1,l1)$
will cause $\m^\ag$ and $\m_1$ to end up in states differing in predicate
$at(p1,l1)$. Hence, given the correct initial state, the solution to the
planning problem $P_{PO}$ will give the correct distinguishing plan.

\begin{theorem}
\label{thm:plan_outcome}
Given a pair of models $\mathcal{M}_i$ and $\mathcal{M}_j$, the planning
problem $P_{PO}$ has a solution iff $\mathcal{M}_i$ and $\mathcal{M}_j$
have a distinguishing plan outcome query $\mathcal{Q}_{PO}$.
\end{theorem}

\begin{sproof}
The input to the planning problem $P_{PO}$ consists of an initial state
$s_\mathcal{I}$.
If the planner can solve $P_{PO}$ with initial state $s_\mathcal{I}$
to give a plan $\pi$, the distinguishing query is a
combination of $s_\mathcal{I}$ and $\pi$. Similarly, if $\m_i \dist^{\q_{PO}} \m_j$,
then giving the initial state $s_\mathcal{I}$ as part of planning problem $P_{PO}$, the
plan $\pi$ will be a solution which is part of $\q_{PO}$.
\end{sproof}

\subsection{Filtering Possible Models}
\label{sec:filter_mods}
This section describes the \textit{filter\_models()} module in
Algorithm \ref{alg:AIA} which takes as input $\m^\ag$, $\m_i$, $\m_j$,
and the query $\q$ (Sec.\,\ref{sec:query_gen}), and returns the
subset $\m_{prune}$ which is not consistent with $\m^\ag$.

First, AAM \emph{poses the query} $\q$ to $\m_i$, $\m_j$, and the
agent $\ag$. Based on the responses of all three, it determines if the two
models are prunable, i.e., $\m_i \prune \m_j$. As mentioned in
Def.\,\ref{def:prunable}, checking for prunability involves checking if
response to the query $\q$ by one of the models $\m_i$ or $\m_j$ is
consistent with that of the agent or not.
\begin{theorem}
\label{thm:consistency}

Let $\m_i, \m_j \in \{\m_+, \m_-, \m_\emptyset\}$ be the models generated
by adding the pal tuple $\gamma$ to $\m'$ which is an abstraction of the
true agent model $\m^\ag$. Suppose $\q = \langle s_{\mathcal{I}}^\q, \pi^\q \rangle$
is a distinguishing query for two distinct models $\m_i, \m_j$,
i.e. $\m_i \dist^{\q} \m_j$, and the response of models
$\m_i, \m_j,$ and $\m^\ag$ to the query $\q$ are
$\q(\m_i) = \langle \ell^i,\langle p^i_1,\dots, p^i_m \rangle \rangle$,
$\q(\m_j) = \langle \ell^j,\langle p^j_1,\dots, p^j_n \rangle \rangle$,
and
$\q(\m^{\ag}) = \langle \ell^\ag, \langle p^\ag_1,\dots,p^\ag_k\rangle \rangle$.
When $\ell^\ag = len(\pi^\q)$, $\m_i$ is not an abstraction of $\m^{\ag}$
if $len(\pi^\q) \not= \ell^i$ or $\{ p^i_1,\dots, p^i_m \}
\not\subseteq \{ p^\ag_1,\dots,p^\ag_k\}$.

\end{theorem}
\begin{sproof}
Proving by induction, the base case is adding a single pal tuple
$\langle p,a,l \rangle$ to an empty model (which is a consistent
abstraction of $\m^\ag$) resulting in 3 models. The 2 models pruned
based on Def.\,\ref{def:consistency} can be shown to be inconsistent
with $\m^\ag$, leaving out the one consistent model. For the
inductive step, it can be shown that after adding a pal tuple to a
consistent model it is not consistent with $\m^\ag$ only if it does
not execute the full plan (the precondition is inconsistent), or
if the end state reached by the model is not a subset of the state
of the agent (the effect is inconsistent).
\end{sproof}

If the models are prunable, then the palm tuple being added in the
inconsistent model cannot appear in any model consistent with
$\mathcal{A}$. As we discard such palm tuples at abstract levels
(as depicted in Fig.\,\ref{fig:lattice_iaa} (a)), we prune out a
large number of models down the lattice (as depicted in
Fig.\,\ref{fig:lattice_iaa} (c)), hence we keep the intractability
of the approach in check and end up asking less number of queries.

\subsection{Updating PAL ordering}
\label{sec:update_ordering}

This section describes the \textit{update\_pal\_ordering()} module
in AIA (line 15). It is called when the query generated by
\textit{generate\_query()} module is not executable by $\ag$, i.e.,
$len(\pi^\q) \not= \ell^\ag$. E.g.,
consider two abstract models $\m_2$ and $\m_3$ being considered by
AAM (Fig.\,\ref{fig:load_truck}). At this level of abstraction,
AAM does not have knowledge of the predicate $at(?p,?l)$, hence
it will generate a plan outcome query with initial state
$at(?t,?l)$ and plan $\emph{load\_truck}(p1,t1,l1)$
to distinguish between $\m_2$ and $\m_3$. But this cannot be
executed by the agent $\ag$ as its precondition $at(?p,?l)$ is not
satisfied, and hence we cannot discard any of the models.

Recall that in response to the plan outcome query we get the failed
action $a_{\f}$ $= \pi{[\ell\!+\!1]}$ and the final state $s_{\f}$.
Since the query plan $\pi$ is generated using $\m_i$ and $\m_j$
(which differ only in the newly added palm tuple), they both
would reach the same state $\overline{s}_{\f}$ after executing
first $\ell$ steps of $\pi$. Thus, we search $\mathbb{S}$ for a
state $s \supset \overline{s}_{\f}$ where $\ag$ can execute $a_{\f}$. 
Similar to \citet{Stern2017}, we
infer that any predicate which is false
in $s$ will not appear in $a_{\f}$'s precondition
in the positive mode.
Next, we iterate through the set of predicates $p' \subseteq s \setminus
\overline{s}_{\f}$ and add them to $\overline{s}_{\f}$ to check if
$\ag$ can still execute $a_{\f}$.  Thus, on adding a predicate $p \in p'$ to
the state $\overline{s}_{\f}$,
if $\ag$ cannot execute $a_{\f}$, we add $p$ in negative mode
in $a_{\f}$'s precondition, otherwise in $\emptyset$ mode. All pal
tuples whose modes are correctly inferred in this way are therefore
removed from the pal ordering.

\subsubsection{Equivalent Models} It is possible for AIA to encounter a
pair of models $\m_i$ and $\m_j$ that are not prunable.
In such cases, the models $\m_i$ and $\m_j$ are functionally
equivalent and none of them can be discarded. Hence, both the
models end up in the set $poss\_models$ in line 18 of AIA.

\subsection{Correctness of Agent Interrogation Algorithm}
In this section, we prove that the set of estimated models returned
by AIA is correct and the returned models are functionally
equivalent to the agent's model, and no correct model is discarded
in the process.

\begin{theorem}
\label{thm:terminate}
The Agent Interrogation Algorithm (algorithm \ref{alg:AIA}) will
always terminate and return a set of models, each of which are
functionally equivalent to the agent's model $\m^\ag$.
\end{theorem}
\begin{sproof}
Theorem \ref{thm:plan_outcome} and Theorem \ref{thm:consistency}
prove that whenever we get a prunable query, we discard only
inconsistent models, thereby ensuring that no correct model is
discarded. When we do not get a prunable query, we infer the
correct precondition of the failed action using
\textit{update\_pal\_ordering()}, hence the number of refined
palm tuples always increases with the number of iterations of
AIA, thereby ensuring its termination in finite time.
\end{sproof}

{\section{Empirical Evaluation}
\label{sec:experiments}}

\begin{table}[t]
    \footnotesize
    \centering
    \footnotesize
    \begin{tabular}{@{}|l|c|c|c|c|c|@{}}
        \hline
        \textbf{Domain} & $\mathbf{|\mathbb{P}^*|}$ &$\mathbf{|\mathbb{A}|}$ & $\mathbf{|\hat{\q}|}$ & $\mathbf{t_\mu}$ \textbf{(ms)} & $\mathbf{t_\sigma}$ \textbf{($\mathbf{\mu}$s})\\
        \hline
        \hline
        Gripper      & 5   & 3  & 17  & 18.0 & 0.2 \\
        \hline
        Blocksworld & 9   & 4  & 48  & 8.4 & 36\\
        \hline
        Miconic     & 10  & 4  & 39  & 9.2 & 1.4\\
        \hline
        Parking      & 18  & 4  & 63  & 16.5 & 806\\
        \hline
        Logistics    & 18  & 6  & 68  & 24.4 & 1.73\\
        \hline
        Satellite    & 17  & 5  & 41  & 11.6 & 0.87\\
        \hline
        Termes       & 22  & 7  & 134 & 17.0 & 110.2\\
        \hline
        Rovers      & 82   & 9  & 370  & 5.1 &  60.3 \\
        \hline
        Barman       & 83   & 17  & 357  & 18.5  & 1605 \\
        \hline
        Freecell     & 100 & 10 & 535 & 2.24$^\dagger$ & 33.4$^\dagger$ \\
        \hline
    \end{tabular}
    \caption{The number of queries ($|\hat{\q}|$), average time per query ($t_\mu$), and variance of time per query ($t_\sigma$)  generated by AIA with FD. Average and variance are calculated for 10 runs of AIA, each on a separate problem.
    {\small $^\dagger$Time in sec. }}
    \label{tab:domain_var}
\end{table}

We implemented AIA in Python to evaluate the efficacy of our
approach.\footnote{Code available at https://git.io/Jtpej} In this
implementation, initial states ($\mathbb{S}$, line 1 in Algorithm
\ref{alg:AIA}) were collected by making the agent perform random walks in
a simulated environment. We used a maximum of 60 such random initial
states for each domain in our experiments. The implementation assumes that
the domains do not have any constants and that actions and predicates do
not use repeated variables (e.g., $at(?v, ?v))$, although these assumptions
can be removed in practice without affecting the correctness of algorithms.
The implementation is optimized to store the agent’s answers to queries;
hence the stored responses are used if a query is repeated.

We tested AIA on two types of agents: symbolic agents that use models from
the IPC (unknown to AIA), and simulator agents that report states
as images using PDDLGym. We wrote image classifiers for each predicate for
the latter series of experiments and used them to derive state
representations for use in the AIA algorithm. All experiments were executed
on 5.0 GHz Intel i9-9900 CPUs with 64 GB RAM running Ubuntu 18.04.

The analysis presented below shows that AIA learns the correct model with
a reasonable number of queries, and compares our results with the closest
related work, FAMA~\cite{aineto2019learning}. We use the metric of
\textit{ model accuracy} in the following analysis: the number of
correctly learnt palm tuples normalized with  the total number of palm
tuples in $\m^\ag$.

\begin{figure}[t]
    \centering
    \includegraphics[width=\columnwidth]{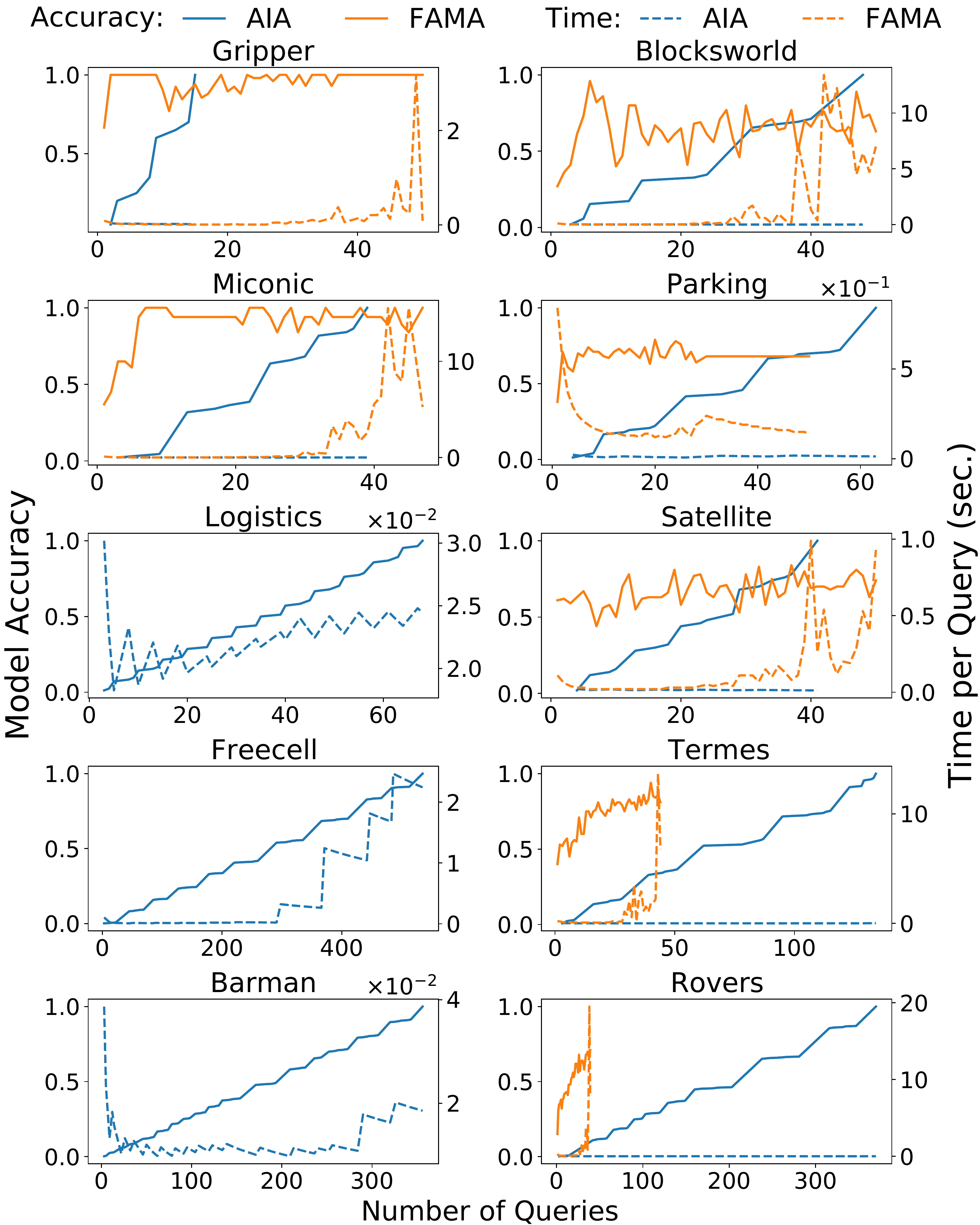}
    \caption{Performance comparison of AIA and FAMA in terms of model
    accuracy and time taken per query with an increasing number of queries.}
    \label{fig:graphs}
\end{figure}

\subsubsection{Experiments with symbolic agents} We initialized the agent
with one of the 10 IPC domain models, and ran AIA on the resulting agent.
10 different problem instances were used to obtain average performance
estimates.

Table \ref{tab:domain_var} shows that the number of queries required
increases with the number of predicates and actions in the domain. We used
Fast Downward~\cite{Helmert06thefast} with LM-Cut
heuristic~\cite{Helmert2009LandmarksCP} to solve the planning problems.
Since our approach is planner-independent, we also tried using
FF~\cite{hoffmann2001} and the results were similar. The low variance shows
that the method is stable across multiple runs.

\subsubsection{Comparison with FAMA} We compare the performance of AIA with
that of FAMA in terms of stability of the models learnt and the time taken
per query. Since the focus of our approach is on automatically generating
useful traces, we provided FAMA randomly generated traces of length 3 (the
length of the longest plans in AIA-generated queries) of the form used
throughout this paper ($\langle s_\mathcal{I},a_1,a_2,a_3,s_G\rangle$).

Fig. \ref{fig:graphs} summarizes our findings. AIA takes lesser time per
query and shows better convergence to the correct model. FAMA sometimes
reaches nearly accurate models faster, but its accuracy continues to
oscillate, making it difficult to ascertain when the learning process
should be stopped (we increased the number of traces provided to FAMA until
it  ran out of memory). This is because the solution to FAMA's internal
planning problem introduces spurious palm tuples in its model if the input
traces do not capture the complete domain dynamics. For Logistics,
FAMA generated an incorrect planning problem, whereas for Freecell and
Barman it ran out of memory (AIA also took considerable time for Freecell).
Also, in domains with negative preconditions like Termes, FAMA was
unable to learn the correct model. We used
Madagascar~\cite{rintanen2014madagascar} with FAMA as
it is the preferred planner for it. We also tried FD and FF with FAMA, but
as the original authors noted, it could not scale and ran out of memory on
all but a few Blocksworld and Gripper problems where it was much slower
than with Madagascar.

\begin{figure}[t]
    \centering
    \begin{subfigure}[t]{.35\columnwidth}
    \centering
    \includegraphics[width=\linewidth]{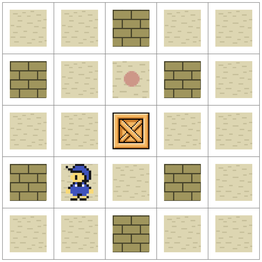}
  \end{subfigure}
  \quad
  \begin{subfigure}[t]{.35\columnwidth}
    \centering
    \includegraphics[width=\linewidth]{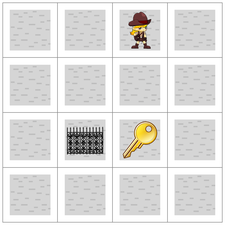}
  \end{subfigure}
  \caption{PDDLGym's simulated Sokoban (left) and Doors (right) 
  environments used for the experiments.}
  \label{fig:sim}
\end{figure}

\subsubsection{Experiments with simulator agents} AIA can also be used with
simulator agents that do not know about predicates and report states as
images. To test this, we wrote classifiers for detecting predicates from
images of simulator states in the PDDLGym~\cite{silver2020pddlgym}
framework. This framework provides ground-truth PDDL models, thereby
simplifying the estimation of accuracy. We initialized the agent with one
of the two PDDLGym environments, Sokoban and Doors shown in
Fig.\,\ref{fig:sim}. AIA inferred the correct model in both cases and the
number of instantiated predicates, actions, and the average number of
queries (over 5 runs) used to predict the correct model for Sokoban were
35, 3, and 201, and that for Doors were 10, 2, and 252.

\section{Conclusion} \label{sec:conclusion}

We presented a novel approach for efficiently learning the internal model of
an autonomous agent in a STRIPS-like form through query answering.
Our theoretical and empirical results showed that the approach works well for
both symbolic and simulator agents.

Extending our predicate classifier to handle noisy state detection, similar
to prevalent approaches using classifiers to detect symbolic
states~\cite{Konidaris14,asai2017classical} is a good direction for future
work. Some other promising extensions include replacing query and response
communication interfaces between the agent and AAM
with a natural language similar to \citet{Lindsay2017}, or learning other
representations like \citet{Zhuo14learnHTN}.

\section*{Acknowledgements} \label{sec:ack}
We thank Abhyudaya Srinet for his help with the implementation. This work was
supported in part by the NSF under grants IIS 1844325, IIS 1942856, and OIA 1936997.

\section*{Ethics Statement}
\label{sec:ethics}

Learning the internal model of an AI agent is one of the main focus areas of
the AI community in the recent past. This work would enable a layperson to
assess such autonomous agents and to verify if they are safe to work with.
This would increase the adaption rate of AI systems, as it would remove the
dependence of systems using AI on experts who could verify the internal working
of the agent.

Our system asks the agent queries and assumes that the agent can be connected
to a simulator to ensure the correctness of responses. Our approach for such
model learning comes with soundness and completeness guarantees. This implies
that it will find the agent model if there exists one, and the model that it
learns will be correct as per the simulations. As in any approach that
uses simulators, this method is susceptible to errors in programming and
simulator design. This can be addressed independently through research on
formal verification of simulators used in AI.

\bibliography{aia_aaai}

\cleardoublepage
\appendix
\setcounter{theorem}{0}

\section{Proofs}

\begin{theorem}
\label{thm:plan_outcome1}
Given a pair of models $\mathcal{M}_i$ and $\mathcal{M}_j$, the planning
problem $P_{PO}$ has a solution iff $\mathcal{M}_i$ and $\mathcal{M}_j$ have a
distinguishing plan outcome query $\mathcal{Q}_{PO}$.
\end{theorem}

\begin{sproof}
    $\mathcal{Q}_{PO}$ comprises of an initial state $s_\mathcal{I}$ and plan 
    $\pi$. The initial state $s_\mathcal{I}$ in $\mathcal{Q}_{PO}$ and $P_{PO}$
    is same. Starting with this initial state, an action becomes a part of the
    plan $\pi$ only when it can be applied in any one or both of the models
    $\mathcal{M}_i$ and $\mathcal{M}_j$. Two cases arise here: 1) if the action can
    be executed in both the models, the effect of both the actions is applied to
    the state and the next action is searched; 2) if the action is
    applicable in only one of the two models, the effect of the
    action is a dummy proposition $p_\psi$ which is also the goal. Thus, as soon
    as an action is executed, that: is either applicable in one of the models but 
    not the other; or 
    gives different resulting states in both the models is found, the resulting
    plan becomes the plan needed by the query $\mathcal{Q}_{PO}$. Hence if the 
    planning problem $P_{PO}$ gives a solution plan $\pi$, then there exists a 
    query $\mathcal{Q}_{PO}$ that consists of $s_\mathcal{I}$ and $\pi$ as input.
    \par
    Moreover, as described previously, whenever there exists a distinguishing
    plan-outcome query, the starting state $s_{\mathcal{I}}$ is part of $\q_{PO}$,
    and the way we generate the $P_{PO}$ problem ensures we will get a plan $\pi$
    as the solution.
\end{sproof}

To prove our next theorem, we will need some additional lemmas that we prove
below:

\begin{lemma}
\label{lem:0pre}
For an empty model $\m$ having 0 palm tuples, when concretizing with a new pal
tuple, a distinguishing query can be found only if $l = \emph{pre}$ in pal tuple.
\end{lemma}
\begin{proof}
Recall that we add the pal tuple
$\gamma = \langle p,a,l \rangle$ in modes $i$ and $j$ to $\m'$ to get $\m_i'$
and $\m_j'$ as follows:\\ 
If the location $l = \emph{eff}$, we add the palm tuple normally 
to $\m'$, i.e., $\m_m' = \m' \cup \langle p , a, l, m \rangle$, where 
$m \in \{i,j\}$. If $l = \emph{pre}$, we add a dummy predicate 
$p_u$ in disjunction with the predicate $p$ to the precondition of both 
the models.
We then modify the models $\m_i'$ and $\m_j'$ 
further in the following way:
\begin{align*}
    \m_m'' = \m_m' \cup \, &\{\langle p_u, a', l',+ \rangle : \forall a',l' \, \, \langle a',l' \rangle \not\in \\ &\{ \!\langle a^*\!\!,l^*\!\rangle\!\!: \exists m^* \, \langle p,a^*,l^*,m^*\rangle \in \m'\}\}\\
    \cup \, &\{\langle p_u, a', l',- \rangle : \forall a',l' \, \, \langle a',l' \rangle \in \\
    &\{ \!\langle a^*\!\!,l^*\!\rangle\!\!: \!l^*\!\!= \!\emph{eff} \land\! \exists m^* \, \!\langle p,a^*,l^*,m^*\rangle \!\!\in\! \m'\}\}
\end{align*}

Since $\m' = \{ \}$, if pal tuple $\gamma = \langle p,a,l=\emph{eff}\rangle$, for
all actions in the model the precondition will be $p_u$. Since $p_u$ is not
present in the initial state, no action is executable.

Whereas if $\m' = \{ \}$, and pal tuple being added is $\gamma = \langle
p,a,l=\emph{pre}\rangle$, precondition of all actions except $a$ will be $p_u$,
and precondition of $a$ will be $p \lor p_u$. Hence, a distinguishing query is
possible such that the plan $\pi$ of the distinguishing query contains $a$.
\end{proof}

\begin{lemma}
\label{lem:last_action}
Let $\m_i, \m_j \in \{\m_+, \m_-, \m_\emptyset\}$ be the models generated by
adding pal tuple $\gamma = \langle p,a,l \rangle$ to $\m'$. Suppose $\q = \langle
s_\mathcal{I}^\q, \pi^\q \rangle$ is a distinguishing query for two distinct models $\m_i,
\m_j$, i.e. $\m_i \dist^{Q} \m_j$. The last action in the plan $\pi^\q$ will be
$a$.
\end{lemma}
\begin{proof}
We prove this by contradiction. Assume that the last action of the distinguishing
plan $\pi^\q$ is $a' \neq a$. Now the query $\q$ used to distinguish between
$\m_i$ and $\m_j$ is generated using the planning problem $P_{PO}$ which has a
solution if both the models have different precondition or different effect for
the same action. Since the last action of the plan is $a'$, the two models either
have different preconditions for $a'$ or different effects. This is not possible
as $\m_i$ and $\m_j$ differ only in one palm tuple having action $a$. Hence
$a'=a$.
\end{proof}

\begin{lemma}
\label{lem:penultimate}
Let $\m_i, \m_j \in \{\m_+, \m_-, \m_\emptyset\}$ be the models generated by
adding pal tuple $\gamma$ to $\m'$ which is an abstraction of $\m^\ag$. Suppose
$\q = \langle s_\mathcal{I}^\q, \pi^\q \rangle$ is a distinguishing query for two distinct
models $\m_i, \m_j$, i.e. $\m_i \dist^{Q} \m_j$. $\q(\m^{\ag}) = \langle \ell^\ag,
\langle p^\ag_1,\dots,p^\ag_k\rangle \rangle$ and $\ell^\ag = len(\pi^\q)$. Let
$\q_{1\dots z}$ represent the query with same initial state $s_\mathcal{I}^\q$, and plan
$\pi^\q[1:z]$, i.e. first $z$, ($z<len(\pi^\Q)$) actions from plan $\pi^\q$. The
response of models $\m_i, \m_j,$ and $\m^A$ to the query $\q_{1\dots z}$ are:
$\q_{1\dots z}(\m_i) = \langle \ell^i,\langle p^i_1,\dots, p^i_m \rangle \rangle$,
$\q_{1\dots z}(\m_j) = \langle \ell^j,\langle p^j_1,\dots, p^j_n \rangle \rangle$,
and 
$\q_{1\dots z}(\m^{\ag}) = \langle \ell^\ag, \langle p^\ag_1,\dots,p^\ag_h\rangle
\rangle$. if $ z=len(\pi^\q)-1$, then $m=n$, $\{p^i_1,\dots, p^i_m\} =
\{p^j_1,\dots, p^j_n\}$ and $\{p^i_1,\dots, p^i_m\} \setminus \{p_u\} \subseteq
\{p^\ag_1,\dots,p^\ag_h\}$.
\end{lemma}
\begin{proof}
The query $\q$ used to distinguish between $\m_i$ and $\m_j$ is generated using the planning problem $P_{PO}$ which has a solution if both the models have different preconditions or different effects for the same action. 
Since $z<len(\pi^\q)$, all the intermediate states generated during plan execution should also be the same for $\m_i$ and $\m_j$. Since $\m'$ is an abstraction of $\m^\ag$, all the palm tuples already present in it are also present in $\m^\ag$. The only new palm tuples in $\m_i$ and $\m_j$ are the ones involving pal tuple $\gamma$ or involving predicate $p_u$. Thus, if $ z=len(\pi^\q)-1$, then $m = n$ and $\{p^i_1,\dots, p^i_m\} = \{p^j_1,\dots, p^j_n\}$. Now, if $l = \emph{pre}$ and $z = len(\pi^\q)$, then $p_u \not\in \{p^i_1,\dots, p^i_m\}$, since $a$'s precondition has $(p \lor p_u)$ in conjunction with other preconditions, and if $p_u \in \{p^i_1,\dots, p^i_m\}$, then $a$ wouldn't be the distinguishing (last) action in $\pi^\q$. If $l = \emph{eff}$, then $p_u \in \{p^i_1,\dots, p^i_m\}$ or $p_u \not\in \{p^i_1,\dots, p^i_m\}$. If $p_u \in \{p^i_1,\dots, p^i_m\}$, then $\{p^i_1,\dots, p^i_m\} \setminus p_u \subseteq \{p^\ag_1,\dots,p^\ag_h\}$, as $\m'$ is an abstraction of $\m^\ag$. The same holds true if $p_u \not\in \{p^i_1,\dots, p^i_m\}$.
\end{proof}

\begin{theorem}
\label{thm:consistency1}
Let $\m_i, \m_j \in \{\m_+, \m_-, \m_\emptyset\}$ be the models generated
by adding the pal tuple $\gamma$ to $\m'$ which is an abstraction of the
true agent model $\m^\ag$. Suppose $\q = \langle s_{\mathcal{I}}^\q, \pi^\q \rangle$
is a distinguishing query for two distinct models $\m_i, \m_j$,
i.e. $\m_i \dist^{\q} \m_j$, and the response of models
$\m_i, \m_j,$ and $\m^\ag$ to the query $\q$ are
$\q(\m_i) = \langle \ell^i,\langle p^i_1,\dots, p^i_m \rangle \rangle$,
$\q(\m_j) = \langle \ell^j,\langle p^j_1,\dots, p^j_n \rangle \rangle$,
and
$\q(\m^{\ag}) = \langle \ell^\ag, \langle p^\ag_1,\dots,p^\ag_k\rangle \rangle$.
When $\ell^\ag = len(\pi^\q)$, $\m_i$ is not an abstraction of $\m^{\ag}$
if $len(\pi^\q) \not= \ell^i$ or $\{ p^i_1,\dots, p^i_m \}
\not\subseteq \{ p^\ag_1,\dots,p^\ag_k\}$.
\end{theorem}
\begin{proof}
We prove this by mathematical induction. Let P(n) be the proposition that for every model with $n$ pal tuples, 
which is consistent with $\m^\ag$, refining it with a pal tuple with the correct mode according to Def. 3 will prune out the inconsistent models. \\
\\
\mysssection{Base Case:} The proof for P(0) being true is by case analysis. Assume the model $\m' = \{\}$, which is consistent with $\m^\ag$, is concretized with pal tuple $\gamma = \langle p,a,l \rangle$. There are two cases:\\

\noindent \textit{Case 1}: Consider $l=\emph{pre}$. This case splits into 2 subcases:\\

\noindent \textit{Case 1.1}: If $p \in s_\mathcal{I}^\q$, $\pi^\q = \langle a \rangle$,  and $\ell^\ag = len(\pi^\q) = 1$, then $\langle p,a,\emph{pre}, - \rangle \not\in \m^\ag$. Also $\m_j \dist^\q \m_-$, where $j \in \{+,\emptyset\}$, as $\ell^j = 1$, and $\ell^- = 0$. Hence P(0) is true.\\

\noindent \textit{Case 1.2}: If $\neg p \in s_\mathcal{I}^\q$, $\pi^\q = \langle a \rangle$,  and $\ell^\ag = len(\pi^\q) = 1$, then $\langle p,a,\emph{pre}, + \rangle \not\in \m^\ag$. Also $\m_j \dist^\q \m_+$, where $j \in \{-,\emptyset\}$, as $\ell^j = 1$, and $\ell^+ = 0$. Hence P(0) is true.\\

\noindent \textit{Case 2}: Consider $l=\emph{eff}$. If $\m = \{\}$ and $l=\emph{eff}$, $\forall i,j \in \{+,-,\emptyset\}$, $\not\exists \q \,\,\m_i \dist^\q \m_j$. Hence P(0) is true.\\

\mysssection{Inductive Step:} Assume that P(n) is true for some $n \geq 0$; that is we have a model $\m'$, with $n$ palm tuples, which is an abstraction of $\m^\ag$, and refining it with a pal tuple $\gamma = \langle p,a,l \rangle$ will generate models with $n+1$ tuples. From Lemma \ref{lem:penultimate}, we know that before executing the last action, the state reached by both the abstracted models ($\overline{s}_{F-1}$) will be a subset of the state reached by $\m^\ag$ ($s_{F-1}$). There are two cases:\\

\noindent \textit{Case 1}: Consider $l=\emph{pre}$. Since $l=\emph{pre}$, $p_u \not\in \overline{s}_{F-1}$. This case splits into 2 subcases:\\

\noindent \textit{Case 1.1}:  If $p \in \overline{s}_{F-1}$,  and $\ell^\ag = len(\pi^\q)$, then $\langle p,a,\emph{pre}, - \rangle \not\in \m^\ag$. Also $\m_j \dist^\q \m_-$, where $j \in \{+,\emptyset\}$, as $\ell^j = len(\pi^\q)$, and $\ell^- = len(\pi^\q) -1 $. Thus, $\m_-$ is not an abstraction of $\m^\ag$. Hence P(n) is true.\\

\noindent \textit{Case 1.2}: If $\neg p \in \overline{s}_{F-1}$, and $\ell^\ag = len(\pi^\q)$, then $\langle p,a,\emph{pre}, + \rangle \not\in \m^\ag$. Also $\m_j \dist^\q \m_+$, where $j \in \{-,\emptyset\}$, as $\ell^j = len(\pi^\q)$, $\ell^+ = len(\pi^\q) -1 $.  Thus, $\m_+$ is not an abstraction of $\m^\ag$. Hence P(n) is true.\\

\noindent \textit{Case 2}: Consider $l=\emph{eff}$. Since $l=\emph{eff}$, $p_u$ may or may not be in $\overline{s}_{F-1}$. In either case, the full plan is executed in $\m_i,\m_j$ and $\m^\ag$. Hence we can compare the states reached after executing the complete plan. 
Let $\overline{s}_F^{\m_i} = \{p^i_1,\dots, p^i_m\}$, $\overline{s}_F^{\m_j} = \{p^i_1,\dots, p^i_n\}$, and $s_F = \{p^i_1,\dots, p^i_k\}$ 
be the final states reached upon executing $\pi^\q$ in $\m_i,\m_j$ 
and $\m^\ag$ respectively and $\overline{s}_{F-1}$
is the state reached in $\m_i$ and $\m_j$ before executing action $a$. This case splits into 2 subcases:\\

\noindent \textit{Case 2.1}:   If $p \in \overline{s}_{F-1}$. If $p \in s_F$, $\langle p,a,\emph{eff},- \rangle \not\in \m^\ag$ and $\overline{s}_{F}^{\m_-} \not\subseteq s_F$. Similarly if  $\neg p \in s_F$, $\langle p,a,\emph{eff},+ \rangle \not\in \m^\ag$ and $\overline{s}_{F}^{\m_+} \not\subseteq s_F$, and $\langle p,a,\emph{eff},\emptyset \rangle \not\in \m^\ag$ and $\overline{s}_{F}^{\m_\emptyset} \not\subseteq s_F$. Hence P(n) is true.\\

\noindent \textit{Case 2.2}:   If $\neg p \in \overline{s}_{F-1}$. If $\neg p \in s_F$, $\langle p,a,\emph{eff},+ \rangle \not\in \m^\ag$ and $\overline{s}_{F}^{\m_+} \not\subseteq s_F$. Similarly if  $p \in s_F$, $\langle p,a,\emph{eff},- \rangle \not\in \m^\ag$ and $\overline{s}_{F}^{\m_-} \not\subseteq s_F$, and $\langle p,a,\emph{eff},\emptyset \rangle \not\in \m^\ag$ and $\overline{s}_{F}^{\m_\emptyset} \not\subseteq s_F$. Hence P(n) is true.

This proves that if we add a pal tuple to a model that is an abstraction of $\m^\ag$, then we prune only inconsistent models $\m_i$ whenever $len(\pi^\q) \not= \ell^i$ or $\{ p^i_1,\dots, p^i_m \} \not\subseteq \{ p^\ag_1,\dots,p^\ag_k\}$ when $\ell^\ag = len(\pi^\q)$.
\end{proof}

To prove our next theorem we'll need some additional lemmas that we prove below:
\begin{lemma}
\label{lem:incorrect_models}
    If we prune away an abstract model $\mathcal{M}^{abs}$, then no possible concretization of $\mathcal{M}^{abs}$ will result into a model consistent with the agent model $\mathcal{M}^\mathcal{A}$.
\end{lemma}
\begin{sproof}
At each node in the lattice, we always prune away some of the models. If we discard an inconsistent model, it is because some palm tuple in the model has different mode $m$, than that of $\m^\ag$. This inorrrect palm tuple will also be present in all its concretizations, making all of them inconsistent with $\m^\ag$. Theorem \ref{thm:consistency1} proves that we always prune the inconsistent models at each node, hence none of their concretizations will result into a model consistent with $\m^\ag$.
\end{sproof}

With the guarantee that we are not pruning away any correct possible model, we now prove that the agent interrogation algorithm will terminate, hence giving a solution. 

\begin{lemma}
\label{lem:terminate}
The Agent Interrogation Algorithm (AIA) will always terminate.
\end{lemma}
\begin{sproof}
At each step of the algorithm, when we consider a refinement in terms of pal tuples, we are left with one or more variants of the pal tuple. This ensures that we never have to refine the models more than once at a single level in the lattice. In our subset lattice, "level" is equivalent to the number of refined pal tuples. Since we refine at least one pal tuple in every iteration of the algorithm, the algorithm is bound to terminate as the number of pal tuples is finite for a finite number of propositions and actions under consideration.
\end{sproof}

\begin{theorem}
\label{thm:terminate1}
The Agent Interrogation Algorithm (Algorithm 1) will always terminate and return a set of models, each of which are functionally equivalent to agent's model $\mathcal{M}^\mathcal{A}$.
\end{theorem}
\begin{sproof}
Theorem \ref{thm:plan_outcome1}, Theorem \ref{thm:consistency1}, and Lemma \ref{lem:incorrect_models} prove that whenever we get a prunable query, we discard only inconsistent models, thereby ensuring that no correct model is discarded. When we don't get a prunable query, we infer the correct precondition of the failed action using \textit{update\_pal\_ordering()}. Lemma \ref{lem:terminate} shows the the number of refined palm tuples always increase with the number of iterations of the algorithm, thereby ensuring the termination of the algorithm in finite time. 
\end{sproof}
\end{document}